\documentclass[lettterpaper]{article}
\usepackage{hyperref}
\usepackage{amsmath}
\usepackage{amssymb}
\usepackage[toc,page]{appendix}
\usepackage[margin=15pt,font=small,labelfont=sc]{caption}
\usepackage{lipsum}
\usepackage{blkarray}
\usepackage{algorithm}
\usepackage{algpseudocode}

\providecommand{\keywords}[1]{\textbf{Keywords:} #1}

\newcommand{\argmin}{\operatornamewithlimits{argmin}}
\long\def\killtext#1{}

\newtheorem{theorem}{Theorem}[section]

\newtheorem{lemma}[theorem]{Lemma}

\newtheorem{corollary}[theorem]{Corollary}

\newtheorem{remark}[theorem]{Remark}
\newtheorem{exa}[theorem]{Example}

\newenvironment{proof}{\noindent{\bf Proof.}}{\hfill$\square$\medskip \\ }

\newenvironment{proof*}[1]{\noindent{\bf Proof of #1.}}{\hfill$\square$\medskip}

\def\Pr{{\sf P}}
\def\eps{\varepsilon}

\def\BB{\mathcal{B}}

\def\RR{\mathcal{R}}
\def\SS{\mathcal{S}}
\def\VV{\mathcal{V}}

\def\Gb{\mathbf{G}}

\def\fb{\mathbf{f}}
\def\gb{\mathbf{g}}

\def\ub{\mathbf{u}}
\def\vb{\mathbf{v}}

\def\Ebb{\mathbb{E}}

\def\Rbb{\mathbb{R}}

\newcommand{\wh}{\widehat}
\newcommand{\wt}{\widetilde}
\newcommand{\wbar}{\overline}

\providecommand{\inprod}[2]{\left \langle #1, #2 \right \rangle}

\providecommand{\vect}[1]{\text{vect}(#1)}
\newcommand{\rank}{\text{rank}}

\def\eqnok#1{(\ref{#1})}

\numberwithin{equation}{section}

\title{\Large\bf Theoretical properties of the global optimizer of two layer neural network}
\author{
Digvijay Boob \thanks{digvijabb40@gatech.edu, Department of Industrial and Systems Engineering, Georgia Institute of Technology, GA 30332}
\hspace{50pt} Guanghui Lan \thanks{george.lan@isye.gatech.edu, http://pwp.gatech.edu/guanghui-lan, Department of Industrial and Systems Engineering, Georgia Institute of Technology, GA 30332}
}
\begin{document}
\maketitle
\begin{abstract}
In this paper, we study the problem of optimizing a two-layer artificial neural network that best fits a training dataset. We look at this problem in the setting where the number of parameters is greater than the number of sampled points. We show that for a wide class of differentiable activation functions (this class involves ``almost" all functions which are not piecewise linear), we have that first-order optimal solutions satisfy global optimality provided the hidden layer is non-singular. Our results are easily extended to hidden layers given by a flat matrix from that of a square matrix. Results are applicable even if network has more than one hidden layer provided all hidden layers satisfy non-singularity, all activations are from the given ``good" class of differentiable functions and optimization is only with respect to the last hidden layer. We also study the smoothness properties of the objective function and show that it is actually Lipschitz smooth, i.e., its gradients do not change sharply. We use smoothness properties to guarantee asymptotic convergence of $O(1/\text{number of iterations})$ to a first-order optimal solution. We also show that our algorithm will maintain non-singularity of hidden layer for any finite number of iterations.
\end{abstract}


\section{Introduction}
Neural networks architecture has recently emerged as a powerful tool for a wide variety of applications. In fact, they have led to breakthrough performance in many problems such as visual object classification \cite{KSH12}, natural language processing \cite{CW08} and speech recognition \cite{MDH12}. Despite the wide variety of applications using neural networks with empirical success, mathematical understanding behind these methods remains a puzzle. Even though there is good understanding of the representation power of neural networks \cite{B94}, training these networks is hard. In fact, training neural networks was shown to be NP-complete for single hidden layer, two node and sgn$(\cdot)$ activation function \cite{BR88}. The main bottleneck in the optimization problem comes from non-convexity of the problem. Hence it is not clear how to train them to global optimality with provable guarantees.\\
Neural networks have been around for decades now. A sudden resurgence in the use of these methods is because of the following: Despite the worst case result by \cite{BR88}, first-order methods such as gradient descent and stochastic gradient descent have been surprisingly successful in training these networks to global optimality. For example, \cite{CSMBO16} empirically showed that sufficiently over-parametrized networks can be trained to global optimality with stochastic gradient descent.
\\
Neural networks with zero hidden layers are relatively well understood in theory. In fact, several authors have shown that for such neural networks with monotone activations, gradient based methods will converge to the global optimum for different assumptions and settings \cite{SYM17, HLS15, KKKS11, KS09}.
\\
Despite the hardness of training the single hidden layer (or two-layer) problem, enough literature is available which tries to reduce the hardness by making different assumptions. E.g., \cite{CHMGY14} made a few assumptions to show that every local minimum of the simplified objective is close to the global minimum. They also require some independent activations assumption which may not be satisfied in practice. For the same shallow networks with (leaky) ReLU activations, it was shown in \cite{SC16} that gradient descent can attain global minimum of the modified loss function, instead of the original objective function. Under the same setting, \cite{XLS16} showed that critical points with large ``diversity" are near global optimal. But ensuring such conditions algorithmically is difficult.\\
All the theoretical studies have been largely focussed on ReLU activation but other activations have been mostly ignored. In our understanding, this is the first time a theoretical result will be presented which shows that for almost all nonlinear activation functions including softplus, a first-order optimal solution is also the global optimal provided certain ``simple" properties of hidden layer. Moreover, we show that a stochastic algorithm will give us those required simple properties for free for all finite number of iterations. Our assumption on data distribution is very general and can be reasonable for practitioners. This comes at the cost that the hidden layer of our network can not be wider than the dimension of the input data, say $d$. Since we also look at this problem in over-parametrized setting (where there is hope to achieve global optimality), this constraint on width puts a direct upper-bound of $d^2$ on the number of data points that can be trained. Even though this is a strong upper bound, recent results from margin bounds \cite{NBMS17} show that if optimal network is closer to origin then we can get an upper bound on number of samples independent of dimension of the problem which will ensure closeness of population objective and training objective.
\\
We also show for the first time that even though the objective function for training neural networks is non-convex, it is Lipschitz smooth meaning that gradient of the objective function does not change a lot with small changes in underlying variable. This allows us to show convergence result for the gradient descent algorithm, enabling us to establish an upper bound on the number of iterations for finding an $\eps$-approximate first-order optimal solution ($\| \nabla f() \| \le \eps$). Therefore our algorithm will generate an $\eps$-approximate first-order optimal solution which satisfies aforementioned properties of the hidden layer. Note that this does not mean that the algorithm will reach the global optimal point asymptotically. We discuss technical difficulties to prove such a conjecture in more detail in section 5 which details our convergence results.
\\
At this point we would also like to point that there is good amount of work happening on shallow neural networks. In this literature, we see variety of modelling assumptions, different objective functions and local convergence results. \cite{LY17} focuses on a class of neural networks which have special structure called ``Identity mapping". They show that if the input follows from Gaussian distribution then SGD will converge to global optimal for population objective of the ``identity mapping" network. \cite{BG17} show that for isotropic Gaussian inputs, with one hidden layer ReLU network and single non-overlapping convolutional filter, all local minimizers are global hence gradient descent will reach global optimal in polynomial time for the population objective. For the same problem, after relaxing the constraint of isotropic Gaussian inputs, they show that the problem is NP-complete via reduction from a variant of set splitting problem. In both of these studies, the objective function is a population objective which is significantly different from training objective in over parametrized domain. In over-parametrized regime, \cite{SJL17} shows that for the training objective with data coming from isotropic Gaussian distribution, provided that we start close to the true solution and know maximum singular value of optimal hidden layer then corresponding gradient descent will converge to the optimal solution. This is one of its kind of result where local convergence properties of the neural network training objective function have studied in great detail.
\\
Our result differ from available current literature in variety of ways. First of all, we study training problem in the over-parametrized regime. In that regime, training objective can be significantly different from population objective. Moreover, we study the optimization problem for many general non-linear activation functions. Our result can be extended to deeper networks when considering the optimization problem with respect to outermost hidden layer. We also prove that stochastic noise helps in keeping the aforementioned properties of hidden layer. This result, in essence, provides justification for stochastic gradient descent.
\\
Another line of study looks at the effect of over-parametrization in the training of neural networks \cite{HV15, NH17}. These result are not for the same problem as they require huge amount of over-parametrization. In essence, they require the width of the hidden layer to be greater than number of data points which is unreasonable in many settings. These result work for fairly general activations as do our results but we require a moderate over-parametrization, width $\times$ dimension $\ge$ number of data population, much more reasonable in practice as pointed before from margin bound results. They also work for deeper neural network as do our results when optimization is with respect to outermost hidden layer (and aforementioned technical properties are satisfied for all hidden layers).

\section{Notation and Problem of Interest}
We define set $[q] := \{1, \dots, q\}$. For any matrix $A  \in \Rbb^{a \times b}$, we write $\text{vect}(A) \in \Rbb^{ab \times 1}$ as vector form of the matrix $A$. For any vector $z \in \Rbb^k$, we denote $h(z):= \big[ h(z[1]), \dots, h(z[k]) \big]^T$, where $z[i]$ is the $i$-th element in vector $z$. $\BB^i(r)$ represents a $l_i$-ball of radius $r$, centred at origin. We define component-wise product of two vectors with operator $\odot$.\\
 We say that a collection of vectors, $\{v^i\}_{i=1}^{N} \in \Rbb^d$, is full rank if $\rank \Big( \begin{bmatrix}
 v^1 & \dots & v^N
 \end{bmatrix} \Big) = \min\{d, N\}$. Similarly, we say that collection of matrices, $\{M_i\}_{i=1}^{N} \in \Rbb^{n \times d}$, is full rank if $\rank\Big(  \begin{bmatrix} \vect{M_1} &\dots  &\vect{M_k} \end{bmatrix} \Big)= \min\{N, nd\}$.
\\
A fully connected two-layer neural network has three parameters: hidden layer $W$, output layer $\theta$ and activation function $h$. For a given activation function, $h$, we define neural network function as
\[\phi_{W, \theta}(u) := \theta^T h(Wu).\]
In the above equation, $W \in \Rbb^{n \times d}$ is hidden layer matrix, $\theta \in \Rbb^n$ is the output layer. Finally $h: \Rbb \to \Rbb$ is an activation function.\\ 
The main problem of interest in this paper is the two-layer neural network problem given by
\begin{equation}
\label{SAA_DL}
\min\limits_{\substack{W \in \Rbb^{n \times d}\\ \theta \in \Rbb^n }} f(W, \theta) :=\frac{1}{2N}\sum\limits_{i=1}^{N}( v^i-\phi_{W, \theta}(u^i))^2.
\end{equation}
In this paper, we assume that $(u^i,v^i) \in \Rbb^d \times \Rbb, i \in [N]$ are independently distributed data point and each $u^i$ is sampled from a $d$-dimensional Lebesgue measure.

\section{The basic idea and the Algorithm}
First-order optimality condition for the problem defined in (\ref{SAA_DL}), with respect to ${W}[j,k]$ (j-th row, k-th column element of matrix $W$) $\forall\ j\in [n],\forall \ k \in [d]$ is 
\begin{equation}\label{eqW1}
\nabla_{W}f(W,\theta)[j,k] = \frac{1}{N}\sum\limits_{i=1}^{N}\{v^i-\theta^Th(Wu^i)\} h'(W[j,:]u^{i}) \theta[j]u^i[k] = 0.
\end{equation} 
Equation (\ref{eqW1}) is equivalent to
\begin{equation}
\sum\limits_{i=1}^{N} \{ v^i - \theta^Th(Wu^i)\}\big(h'(Wu^i)\odot \theta \big)u^{i^T} = \boldsymbol{0}.
\end{equation}
(\ref{eqW1}) can also be written in a matrix vector product form:
\begin{equation} 
\label{first_matrix_eq}
Ds = \boldsymbol{0},
\end{equation}
where 
\[
D:= \begin{bmatrix}
h'(W[1,:]u^1)\theta[1]u^1 &\dots  &h'(W[1,:]u^N)\theta[1]u^N\\
\vdots &\ddots &\vdots\\
h'(W[d,:]u^1)\theta[d]u^1 &\dots  &h'(W[d,:]u^N)\theta[d]u^N
\end{bmatrix} \text{and }
s := \begin{bmatrix}
v^1-\theta^Th(Wu^1) \\ \vdots \\ v^N-\theta^Th(Wu^N)
\end{bmatrix}.
\] 
Notice that if matrix $D \in \Rbb^{nd \times N}$ is of full column rank (which implies $nd \ge N$, i.e., number of samples is less than number of parameters) then it immediately gives us that $s =0$ which means such a stationary point is global optimal. This motivates us to investigate properties of $h$ under which we can provably keep matrix $D$ full column rank and develop algorithmic methods to help  maintain such properties of matrix $D$.\\
For the rest of the discussion, we will assume that $n = d$ (our results can be extended to case $n \le d$ easily) and hence $W$ is a square matrix. In this setting, we develop the following algorithm whose output is a provable first-order approximate solution. Here we present the algorithm and in next sections we will discuss conditions that are required to satisfy full rank property of matrix $D$ as well as convergence properties of the algorithm.\\
In the algorithm, we use techniques inspired from alternating minimization to minimize with respect to $\theta$ and $W$. For minimization with respect to $\theta$, we add gaussian noise to the gradient information. This will be useful to prove convergence of this algorithm. We use randomness in $\theta$ to ensure some ``nice'' properties of $W$ which help us in proving that matrix $D$ generated along the trajectory of the algorithm is full column rank. More details will follow in next section.\\
The algorithm has two loops. An outer loop implements a single gradient step with respect to hidden layer, $W$. For each outer loop iteration, there is an inner loop which optimizes objective function with respect to $\theta$ using a stochastic gradient descent algorithm. In the stochastic gradient descent, we generate a noisy estimated of $\nabla_\theta f(W, \theta)$ as explained below.\\
Let $\xi \in \Rbb^d$ be a vector whose elements are i.i.d. Gaussian random variable with zero mean. Then for a given value of $W$ we define stochastic gradient w.r.t. $\theta$ as follows: 
\begin{equation}\label{stochastic_oracle} 
G_{W}(\theta, \xi) = \nabla_{\theta} f(W, \theta) + \xi.
\end{equation}
Then we know that 
\[\Ebb[G _{W}(\theta, \xi) ] = \nabla_{\theta} f(W, \theta).\]
We can choose a constant $\sigma > 0$ such that following holds 
\begin{equation} \label{eq_sigma}
\Ebb\Big[ \big\| G_{W}(\theta, \xi) - \nabla_{\theta} f(W, \theta) \big \|^2 \Big] \le \sigma^2.
\end{equation}
Moreover, in the algorithm we consider a case where $\theta \in \RR$. Note that $\RR$ can be kept equal to $\Rbb^d$ but that will make parameter selection complicated. In our convergence analysis, we will use \begin{equation}\label{eq_define_RR}\RR := \BB^2(R/2),\end{equation} for some constant $R$, to make parameter selection simpler. We use prox-mapping $P_x : \Rbb^d \rightarrow \RR$ as follows:
\begin{equation}\label{prox_map} P_x(y) = \argmin\limits_{z \in \RR} \inprod{y}{z-x} + \frac{1}{2} \|z-x \|^2. \end{equation}
In case $\RR$ is a ball centred at origin, solution of (\ref{prox_map}) is just projection of $x-y$ on that ball. For case where $\RR = \Rbb^d$ then the solution is quantity $x-y$ itself.
\begin{algorithm}[H]
\caption{SGD-GD Algorithm}\label{SGD_GD}
\begin{algorithmic}
\Procedure{}{}
\State $W_0 \gets \text{ Random }d\times d \text{ matrix}$
\State $\theta_0 \gets \text{ Random }d\text{ vector}$
\State Initialize $N_o$ to predefined iteration count for outer ietaration
\State Initialize $N_i$ to predefined iteration count for inner iteration
\State Begin \textbf{outer iteration}:
\For{$k=0,1,2, \dots, N_o$}
	\State $\wbar{\theta}_1 \gets \theta_k$
	\State Begin \textbf{inner iteration}:
	\For {$i = 1,2, \dots, N_i$}
		\State $\wbar{\theta}_{i+1} \gets P_{\wbar{\theta}_i}(\beta_i G_{W_k}(\wbar{\theta}_i, \xi^k_i))$
		\State $\wbar{\theta}^{av}_{i+1} = \big( \sum\limits_{\tau =1}^{i} \beta_\tau \big)^{-1}\sum\limits_{\tau =1}^{i} \beta_\tau \wbar{\theta}_{\tau+1}$
	\EndFor
	\State $\theta_{k+1} \gets \wbar{\theta}^{av}_{N_i+1}$
	\State $W_{k+1} \gets W_k - \gamma_k \nabla_{W}f(W_k, \theta_{k+1})$
\EndFor
\State \textbf{return} $\{W_{N_o+1};\ \theta_{N_o+1}\}$
\EndProcedure
\end{algorithmic}
\end{algorithm} 
Notice that the problem of minimization with respect to $\theta$ is a convex minimization problem. So we can implement many procedures developed in the Stochastic optimization literature to get the convergence to optimal value \cite{NJLS09}. We are implementing SGD which was developed by \cite{L12}.\\
In the analysis, we note that one does not even need to implement complete inner iteration as we can skip the stochastic gradient descent suboptimally given that we improve the objective value with respect to where we started, i.e.,
\begin{equation}f(W_k, \theta_{k+1}) \le f(W_k, \theta_k).\end{equation} 
In essence, if evaluation of $f$ for every iteration is not costly then one might break out of inner iterations before running $N_i$ iterations. If it is costly to evaluate function values then we can implement the whole SGD for convex problem with respect to $\theta$ as specified in inner iteration of the algorithm above. In each outer iteration, we take one gradient decent step with respect to variable $W$. We have total of $N_o$ outer iterations. So essentially we evaluate $\nabla_\theta f(W, \cdot)$ a total of $N_oN_i$ times and $\nabla_Wf(\cdot, \theta)$ total of $N_o$ times.\\
Overall, this algorithm is new form of alternate minimization, where one iteration can be potentially left suboptimally and other one is only one gradient step.

\section{First-order optimality is enough}
We say that $h: \Rbb \to \Rbb$ satisfy the condition ``\textbf{C1}" if 
\begin{itemize}
\item $\forall \text{ interval } (a,b), \ \nexists\ \{c_1, c_2, c_3\} \in \Rbb^3 \text{ s.t. }$
\begin{align*} 
\{h'(x) &=c_1, \forall x \in (a,b)\}  \textit{ or }  \\
\{(x+ c_2)h'(x) +h(x) &= c_3, \forall x \in (a,b)\} .
\end{align*}
\end{itemize}
One can easily notice that most activation functions used in practice e.g., \begin{itemize} \item \textit{(Softplus)} $h(x) := ln(1+e^x)$, \item \textit{(Sigmoid)} $h(x):= \frac{1}{1+e^{-x}}$, \item \textit{(Sigmoid symmetric)} $h(x):= \frac{1-e^{-x}}{1+e^{-x}}$, \item \textit{(Gaussian)} $h(x):= e^{-x^2}$, \item \textit{(Gaussian Symmetric)} $h(x):= 2e^{-x^2}-1$, \item \textit{(Elliot)} $h(x):= \frac{x}{2(1+|x|)}+0.5$, \item \textit{(Elliot Symmetric)} $h(x):= \frac{x}{1+|x|}$, \item \textit{(Erf)} $h(x):= \frac{2}{\sqrt{\pi}}\int\limits_{0}^{x} e^{-t^2/2}dt$, \item \textit{(Hyperbolic tangent)} $h(x):= \tanh(x)$, \end{itemize}  satisfy the condition $\textbf{C1}$. Note that $h'(x)$ also satisfy condition \textbf{C1} for all of them. In fact, except for very small class of functions (which includes linear functions), none of the continuously differentiable functions satisfy condition \textbf{C1}.\\
We first prove a lemma which establishes that columns of the matrix $D$ (each column is a vector form of $d\times d$ matrix itself) are linearly independent when $W = I_d$ and $h'$ satisfies condition \textbf{C1}. We later generalise it to any full rank $W$ using a simple corollary. The statement of following lemma is intuitive but its proof is technical.\\
\begin{lemma} \label{lemma3.1} Suppose $x^i\in \Rbb^d$ are independently chosen vectors from any d-dimensional Lebesgue measure and let $h: \Rbb \to \Rbb$ be any function that satisfies condition  \textbf{C1} then collection of matrices $h(x^i) x^{i^T}, i \in [N]$ are full rank with measure 1.
\end{lemma}
\begin{proof}
The result is trivially true for d =1, we will show this using induction on d.\\
Define $\vb^{i} := \text{vect}(h(x^{i})x^{i^T}), i \in [N]$. Note that it suffices to prove independence of vector $\vb^i , i \in [N]$ for $N \le d^2$.\\
Now for sake of contradiction assume that $\vb^{i}, i \in [N]$, are linearly dependent with positive joint measure on $x^i, i \in [N]$ which is equivalent to positive measure on individual $x^i,  \forall \ i \in [N]$ due to independence of vectors $x^i$.\\
Since $x^{i}$'s are sampled from Lebesgue measure so positive measure on $x^i,  \forall i \in [N]$, implies there exists a $d$-dimensional volume for each $x^{i}$ such that corresponding $\vb^{i}$ are linearly dependent. We can assume volume to be $d$-dimensional hyper-cuboid $Z^i :=\{ x \in \Rbb^d: a^i < x < b^i\}, \forall \ i \in [N]$ (otherwise we can inscribe a hyper-cuboid in that volume). Notice that since $Z^i$ is a d-dimensional hyper-cuboid so $a^i[k] < b^i[k],  \forall\ i \in [N], \forall \ k \in [d]$. Moreover, for any collection satisfying $x^i \in Z^i$, corresponding collection of vector $\vb^{i}$ are linearly dependent, i.e.,
\begin{equation}\label{eq3.4.1}\vb^{1} = \mu_2\vb^{2} + \dots + \mu_N\vb^{N},  \quad  \text{such that }\forall i \in [N], x^i \in Z^i.\end{equation}
Noticing the definition of $Z^1$, we can choose $\epsilon >0$ s.t. $\widetilde{x}^{1}:=x^{1}+ \epsilon e_1 \in Z^1$. Since we ensure that $\wt{x}^1 \in Z^1$ then by (\ref{eq3.4.1}) we have 
\begin{equation}\label{eq3.4.2}\widetilde{v}^{1}:= \text{vect}(h(\wt{x}^{1})\widetilde{x}^{1^T})=   \mu'_2\vb^{2} + \dots + \mu'_N\vb^{N}.\end{equation}
So using (\ref{eq3.4.1}) and (\ref{eq3.4.2}) we get 
\begin{equation}\label{eq3.4.3}\widetilde{v}^{1}-\vb^{1}= \lambda_2\vb^{2} + \dots + \lambda_N\vb^{N}.\end{equation}
Since $h(x^i)x^{i^T}[j,k] = h(x^{i}[j])x^{i}[k]$, we have $h(x^1)x^{1^T}[j,k] = h(\wt{x}^1)\wt{x}^{1^T}[j,k],  \forall j \in \{2,\dots, d\}, k \in \{2,\dots, d\}$. So we have $(d-1)^2$ components of $\widetilde{v}^{1}-\vb^{1}$ are zero. Let us define:
\[
w^1 = \begin{blockarray}{[c]c\}l}
(x^{1}[1]+\epsilon)h(x^{1}[1]+\epsilon) - x^{1}[1]h(x^{1}[1]) & & \\
\epsilon h(x^{1}[2]) &&\\
\vdots && \\
\epsilon h(x^{1}[d]) &&(2d-1) \\
x^{1}[2](h(x^{1}[1]+\epsilon)-h(x^{1}[1])) &&\\
\vdots &&\\
x^{1}[d](h(x^{1}[1]+\epsilon)-h(x^{1}[1]))&&\\
\end{blockarray}, \quad 
z^1 =\begin{blockarray}{[c]c\}l}
0 & & \\
\vdots & & (d-1)^2\\
0 & &
\end{blockarray},
\]
and notice that $\wt{v}^1-\vb^1 = \begin{bmatrix}
w^1 \\ z^1
\end{bmatrix}$. Since $\epsilon > 0, w^1 \ne \textbf{0}$ with measure 1.\\
Let $y^{i} := x^{i}[2:d]$ then last $(d-1)^2$ equations in (\ref{eq3.4.3}) gives us 
\begin{equation}\label{eq3.4.4} \lambda_2h(y^{2})y^{2^T} + \dots + \lambda_Nh(y^{N})y^{N^T} =z^1 = \textbf{0}\end{equation}
By definition we have $y^i \in \Rbb^{d-1}$ are independently sampled from $(d-1)$-dimensional Lebesgue measure. So by inductive hypothesis, rank of collection of matrices $h(y^i)y^{i^T}, i \in \{2,\dots, N\} = \min\{(d-1)^2,N-1\}$. So if $N-1 \le (d-1)^2$ then $\lambda_2 = \dots = \lambda_N = 0$ with measure 1, then by (\ref{eq3.4.3}) we have $w^1 =\textbf{0}$ with measure 1, which is contradiction to the fact that $w^1 \ne \textbf{0}$ with measure 1. This gives us 
\begin{equation}\label {eq_N_lb}
N > (d-1)^2+1 \end{equation}
Notice that (\ref{eq3.4.4}) in its matrix form can be written as linear system
\begin{equation} \label {eq_null_space}
\begin{bmatrix} 
\text{vect}(h(y^2)y^{2^T}) &\dots &\text{vect}(h(y^N)y^{N^T}) \end{bmatrix} 
\begin{bmatrix}
\lambda_2 \\ \vdots \\ \lambda_N\end{bmatrix} = \textbf{0}\end{equation}
By (\ref{eq_null_space}), we have that vector of $\lambda$'s lies in the null space of the matrix. Finally by inductive hypothesis  and (\ref{eq_N_lb}) we conclude that the dimension of that space is $N-1-(d-1)^2 \{> 0\}$. Let $\ub^1, \dots, \ub^{N-1-(d-1)^2} \in \Rbb^{N-1}$ be the basis of that null space i.e. 
\[\begin{bmatrix} 
\text{vect}(h(y^2)y^{2^T}) &\dots &\text{vect}(h(y^N)y^{N^T}) \end{bmatrix} \ub^j = 0,\quad \forall \ j \in \{1, N-1-(d-1)^2\}\] 
Define $t^i \in \Rbb^{2d-1}$ as: \[t^i := \begin{bmatrix}
x^i[1]h(x^i[1]) \\ \vdots \\ x^i[1]h(x^i[d]) \\ x^i[2]h(x^i[1]) \\ \vdots \\ x^i[d]h(x^i[1])
\end{bmatrix}\] then we can rewrite (\ref{eq3.4.3}) as
\begin{equation}
\begin{bmatrix}
w^1 \\ z^1 
\end{bmatrix} = \begin{bmatrix}
t^2 &\dots &t^N \\
\text{vect}(h(y^2)y^{2^T}) &\dots &\text{vect}(h(y^N)y^{N^T}) 
\end{bmatrix}\begin{bmatrix}
\ub^1 &\dots &\ub^{N-1-(d-1)^2}
\end{bmatrix}\begin{bmatrix}
\wh{\lambda}_2 \\ \vdots \\ \wh{\lambda}_{N-(d-1)^2}
\end{bmatrix}
\end{equation}
which implies that
\begin{equation}\label{eq3.4.5} w^{1} = \wh{\lambda}_2\wh{v}^{2} + \dots + \wh{\lambda}_{N-(d-1)^2}\widehat{v}^{N-(d-1)^2}\end{equation} 
where $\widehat{v}^{i} = \begin{bmatrix}
t^2 &\dots &t^N 
\end{bmatrix} \ub^{i-1}, \ \ i = 2, \dots, N-(d-1)^2$ and $z^1$ part of the equation is already satisfied due to selection of null space. \\
Since $N \le d^2 \Rightarrow N-1-(d-1)^2 \le 2d-2$ then $2d-1$ equations specified in (\ref{eq3.4.5}) are consistent in $\le(2d-2)$ variables. Hence we get linearly dependent equations $\forall x^{1}_1 \in (a^{1}_1, b^{1}_1)$ and $\epsilon$ small enough. Since $x^{2}, \dots, x^{N}$ are kept constant, $\vb^2, \dots, \vb^N$ are constant. So $t^2, \dots, t^N$ are constants and we can choose the same basis of null space $\ub^1, \dots, \ub^{N-1-(d-1)^2}$. Hence we have $\widehat{v}^{2}, \dots, \widehat{v}^{(N-(d-1)^2)}$ are constant. Let us define the set $\SS$ to be the index set of linearly independent rows of matrix $[\widehat{v}^{2}\ \dots\ \widehat{v}^{N-(d-1)^2}]$ and every other row is a linear combination of rows in $\SS$. Since (\ref{eq3.4.5}) is consistent  so the same combination must be valid for the rows of $w^1$. \\
Now if $N \le d^2-1$ then number of variables in (\ref{eq3.4.5}) is $\le2d-3$ but number of equations is $2d-1$, therefore at least two equations are linearly dependent on other equation. This implies last $(2d-2)$ equations then function must be dependent on each other:
\[\epsilon \sum\limits_{j=2}^{d}\alpha_j h(x^{(1)}[j])+\bigg( h(x^{(1)}[1]+\epsilon)-h(x^{(1)}[1]) \bigg)\sum\limits_{j=2}^{d}\beta_jx^{(1)}[j] = 0\]
for some fixed combination $\alpha_j, \beta_j$. If we divide above equation by $\epsilon$ and take the limit as $\epsilon \rightarrow 0$ then we see that $h$ satisfies following differential equation on interval $(a^1_1, b^1_1)$: 
\[h'(x) = c_1\] 
which is a contradiction to the condition \textbf{C1}!\\
Clearly this leaves only one case i.e. $N = d^2$ and $(2d-1)$ equations must satisfy dependency of the following form for all $x^{(1)}_1 \in (a^{(1)}_1, b^{(1)}_1)$:
\begin{align*}(x^{(1)}[1]+\epsilon)h(x^{(1)}[1]+\epsilon) &- x^{(1)}[1]h(x^{(1)}[1])\\ &= \epsilon \sum\limits_{j=2}^{d}\alpha_j h(x^{(1)}[j])+\bigg(h(x^{(1)}[1]+\epsilon)-h(x^{(1)}[1])\bigg)\sum\limits_{j=2}^{d}\beta_jx^{(1)}[j]\end{align*}
Again by similar arguments, the combination is fixed. Let $H(x) = xh(x)$ then dividing above equation by $\epsilon$ and taking the limit as $\epsilon \rightarrow 0$, we can see that $h$ satisfies following differential equation:
\begin{equation}\label{eq3.4.6}H'(x) = c_1 + c_2h'(x) \Rightarrow (x-c_2)h'(x) +h(x)= c_1 \end{equation}
which is again a contradiction to the condition \textbf{C1}\\ 
So we conclude that for $N\le d^2$ there does not exist hyper-cuboids $ Z^i \text{ such that } \text{vol}(Z^i) > 0$ and for all $x^i \in Z^i$, corresponding $\vb^i$ are linearly dependent. So we get rank of collection $\{\vb^i\}_{i=1}^{N} = \min\{N, d^2\}$ with measure 1.
\end{proof}
Now Lemma \ref{lemma3.1} gives us a simple corollary:
\begin{corollary}
\label{corollary3.3}
If W is a nonsingular square matrix and $u^{i} \in \Rbb^d$ is independently sampled from a Lebesgue measure then the collection of matrices $\Big\{h(Wu^{i})u^{i^T}\Big\}_{i=1}^{N}$ is full rank with measure 1.
\end{corollary}
\begin{proof}
Let us define $x := Wu$ be another random variable. Since $W$ is full rank and $u$ has Lebesgue measure $\Rightarrow x$ has Lebesgue measure.\\
Now we claim that the collection $h(Wu^i)u^{i^T}$ is full rank iff the collection $h(x^i)x^{i^T}$ is full rank. This can observed as follows:
\begin{align*}
\sum\limits_{i =1}^{N}\lambda_i h(x^{i})x^{i^T} = 0 &\Leftrightarrow \bigg\{\sum\limits_{i =1}^{N}\lambda_i h(Wu^i)u^{i^T} \bigg\}W^T = 0 \\
&\Leftrightarrow \sum\limits_{i =1}^{N}\lambda_i h(Wu^i)u^{i^T} = 0
\end{align*}
Here the second statement follows from the fact $W$ is a non-singular matrix. \\
Now by lemma \ref{lemma3.1} we have that collection $h(x^i)x^{i^T}$ is linearly independent with measure 1. So $h(Wu^i)u^{i^T}$ is linearly independent with measure 1.\\
Since any rotation is $U$ is a full rank matrix so we have the result.
\end{proof}
This means that if $u^i$ in the Problem \eqnok{SAA_DL} are coming from a Lebesgue measure then by Corollary \ref{corollary3.3} we have $h(Wu^i)u^{i^T}$ will be a full rank collection given that we have maintained full rank property of $W$. Now note that in the first-order condition, given in \eqnok{first_matrix_eq}, row of matrix $D$ are scaled by constant factors $\theta[j]$'s, $j \in [d]$. Notice that we may assume $\theta[j] \ne 0$ because otherwise there is no contribution of corresponding $j$-th row of $W$ to the Problem \eqnok{SAA_DL} and we might as well drop it entirely from the optimization problem. Hence we can rescale rows of matrix $D$ by factor $\frac{1}{\theta[j]}$ without changing the rank. In essence, corollary \ref{corollary3.3} implies that matrix $D$ is full rank when $W$ is full rank. So by our discussion in earlier section, we show that satisfying first-order optimality is enough to show global optimality under condition \textbf{C1}.\\
\begin{remark}\label{rotation_invariance} Due to lemma \ref{lemma3.1} and corollary \ref{corollary3.3} then, rank of collection $h(u^i)u^{i^T}$ is invariant under any rotation.
\end{remark}
\begin{remark} As a result of corollary above one can see that the collection of vectors $h(Wx^i)$ is full rank under the assumption that $W$ is non-singular, $x^i \in \Rbb^d$ are independently sampled from Lebesgue measure and $h$ satisfies condition \textbf{C1}.
\end{remark}
\begin{remark}
Since collection $h(Wu^i)$ is also full rank, we can say that $z^i := h(W_1u^i)$ are independent and sampled from a Lebesgue measure for a non-singular matrix $W_1$. Applying the lemma to $z^i$, we have collection of matrices $g(W_2z^i)z^{i^T}$ are full rank with measure 1 for non-singular $W_2$ and $g$ satisfying condition \textbf{C1}. So we see that for multiple hidden layers satisfying non-singularity, we can apply full rank property for collection of gradients with respect to outermost hidden layer.
\end{remark}
\begin{remark} \label{less_row_remark} If $W\in \Rbb^{n\times d}$ is such that $n \le d$ and $W$ is full row rank, then we can extend its basis to create $W'$ and apply corollary \ref{corollary3.3} to get that $h(W'u^i)u^{i^T}$ is full rank with measure 1. So this implies that $h(Wu^i)u^{i^T}$ must have been full rank with probability 1 otherwise we will have contradiction.
\end{remark}
\begin{remark}
We can extend corollary \ref{corollary3.3} to a general result that $h(Wu^i)u^{i^T}$ has rank $\min\{\rank(W)d, N\}$ with measure 1 by removing dependent rows and using remark \ref{less_row_remark}.
\end{remark}

\section{Convergence results}
Even though we have proved that collection $\big\{ h(Wu^i){u^i}^T \big\}_{i=1}^{N}$ is full rank, we can only apply it to an algorithm which is by design going to output a non-singular matrix as final answer. But deriving such guarantees for just last iteration can be challenging. Hence we rather design an algorithm which gives a non-singular $W$ in every iteration. The SGD step we mentioned before is used precisely to obtain such theoretical guarantees. In Lemma \ref{lemma3.4} below, we provide theoretical guarantee that for any finite number of iterations the hidden layer matrix, $W$, is full rank. Later on, we will also show that overall algorithm will converge to first order approximate solution to the problem \eqnok{SAA_DL}. It should be noted however that this can not guarantee convergence to a global optimal solution. To prove such a result, one needs to analyze the smallest singular value of random matrix $D$, defined in (\ref{first_matrix_eq}). More specifically, we have to show that $\sigma_{\min}(D)$ decreases at the rate slower than the first-order convergence rate of the algorithm so that the overall algorithm converges to the global optimal solution. Even if it is very difficult to prove such a result in theory, we think that such an assumption about $\sigma_{\min}(D)$ is reasonable in practice. One more (probably simpler) approach would be to prove asymptotic convergence without any rate guarantees. In essence, we have to show that as $N_o \rightarrow \infty$ we have $W \rightarrow W^*$ then $W^*$ is non-singular. But here as well, we do not have guarantee over the $\rank(W^*)$ since it is a limiting point of the open set of non-singular matrices which can be singular. Analysis of both these approaches can be challenging.\\
Now we analyze the algorithm. For the sake of simplicity of notation, let us define \begin{equation} \label{eq_def_xi_outer}
\xi^{[k]} := \{\xi^1_{[N_i]}, \dots, \xi^k_{[N_i]} \}\end{equation} and  \begin{equation} \label{eq_def_xi_inner}
\xi^j_{[N_i]} = \{\xi^j_1 \ \dots \ \xi^j_{N_i}\},
\end{equation} 
where $N_i$ is the inner iteration count in Algorithm 1. Essentially $\xi^{[k]}$ contains the record of all random samples used until the $k$-th outer iteration in Algorithm 1 and $\xi^j_{[N_i]}$ contains record of all random samples used in the inner iterations of $j$-th outer iteration. 
\begin{lemma}\label{lemma3.4}$\Pr\{\exists \ \vb \text{ such that } W_k \vb  =\boldsymbol{0} \big| \xi^{[k-1]} \} = 0, \forall \ k \ge 0$, where $W_k$ are matrices generated by Algorithm 1 and measure $\Pr\{.\big| \xi^{[k-1]}\}$ is w.r.t. random variables $\xi^k_{[N_i]}$. \end{lemma}
\begin{proof}
This is true for $k =0$ trivially since we are randomly sampling matrix $W_0$. We now show this by induction on $k$. \\
Recall that gradient of $f(W,\theta)$ with respect to $W$ can be written as $\sum\limits_{i=1}^{N} \{v^i - \theta^T h(Wu^i)\}\big(h'(Wu^i)\odot \theta \big)u^{i^T}$. Notice that in effect, we are multiplying $i$-th row of the rank one matrix $h'(Wu^i)u^{i^T}$ by $i$-th element of vector $\theta$. So this can be rewritten as a matrix product
\[\sum\limits_{i=1}^{N}\{v^i-\theta^T h(Wu^i)\} \Theta h'(Wu^i) u^{i^T},\]
where $\Theta : = \text{diag}\{\theta[i],\ i = 1, \dots, d\}$. So iterative update of the algorithm can be given as 
\[W_{k+1} = W_k - \gamma_k \Theta_{k+1} \nabla_{W}f(W_k, \theta_{k+1}),\quad \forall \ k\ge 0.\]
Notice that given $\xi^{[k]}$, vector $\theta_{k+1}$ and corresponding diagonal matrix $\Theta_{k+1}$ are found by SGD in the inner loop so $\theta_{k+1}$ is a random vector. More specifically, since $\{\xi^{k+1}_i\}_{i=1}^{N_i}$ is sequence of independent $d$-dimensional isotropic Gaussian vectors. Hence the distribution of $\xi^{k+1} =\{\xi^{k+1}_i\}_{i=1}^{N_i}$ induces a Lebesgue measure on random variable $\{ \theta_{k+1} \big| \xi^{[k]}\}$\\
Given $\xi^{[k]}$ then $W_k$ is deterministic quantity.\\
For the sake of contradiction, take any vector $\vb$ that is supposed to be in the null space of $W_{k+1}$ with positive probability.
\begin{align*}
W_{k+1} &= W_k - \gamma_k \nabla_{W}f(W_k, \theta_{k+1}) \nonumber \\
&= W_k - \gamma_k \sum\limits_{i=1}^{N}\Theta_{k+1}(v^i - \theta_{k+1}^T h(W_k u^i) ) h'(W_k u^i) u^{i^T}. \nonumber \\
\Rightarrow W_{k+1} \vb  &=  W_k \vb - \gamma_k \sum\limits_{i=1}^{N}\Theta_{k+1}(v^i - \theta_{k+1}^T h(W u^i) ) h'(W_k u^i) u^{i^T} \vb = 0. \nonumber \\
\Rightarrow W_k \vb &= \Theta_{k+1} \sum\limits_{i=1}^{N}(\lambda_i v^i - r_i^T \theta_{k+1}) h'(W_k u^i) &\text{setting } \lambda_i = \gamma_k (\vb^Tu^{i}), r_i = \lambda_i h(W_k u^i) \nonumber \\
&= \Theta_{k+1} \bigg[ \sum\limits_{i=1}^{N}\lambda_i v^i h'(W_k u^i) - \Big( \sum\limits_{i=1}^{N} h'(W_k u^i) r_i^T \Big) \theta_{k+1} \bigg].  \nonumber \\
\end{align*}
Now the last equation is of the form 
\begin{equation}
b = \Theta_{k+1} [w -M \theta_{k+1}], \label{stochastic_theta}
\end{equation}
where $b = W_k \vb,\ w = \sum\limits_{i=1}^{N}\lambda_i v^i h'(W_k u^i),\ M = \sum\limits_{i=1}^{N} h'(W_k u^i)r^{i^T}$.\\
Suppose we can find such $\theta$ with positive probability. Then we can find hypercuboid $Z := \{x\in \Rbb^d |  a < x < b\}$ such that any $\theta_{k+1}$ in given hypercuboid can solve equation (\ref{stochastic_theta}). By induction we have $b \ne \boldsymbol{0}$. We may assume $b[1] \ne 0$. Then to get contradiction on existence of $Z$, we observe that first equation in (\ref{stochastic_theta}) is:
\begin{equation} \label{eq_quadratic_theta}
b[1] = \theta_{k+1}[1] \Big(w[1] - \sum\limits_{j =2}^{d}M[1,j]\theta_{k+1}[j] \Big) - M[1,1]\theta_{k+1}[1]^2, \quad  \forall \ \theta_{k+1} \in (a,b). 
\end{equation}
Hence if we fix $\theta_{k+1}[i] \in (a[i], b[i]), i = 2, \dots, d$ then \eqnok{eq_quadratic_theta} holds for all $\theta_{k+1}[1] \in (a[1],b[1])$. So we conclude that $b[1] = w[1]+\sum\limits_{j=2}^{d}M[1,j]\theta_{k+1}[j] = M[1,1] = 0$. But $b[1]$ can not be 0. Hence we arrive at a contradiction to the assumption that there existed a hypercuboid $Z$ containing solutions of (\ref{stochastic_theta}).\\
Since measure on $\theta_{k+1}$ was induced by $\{ \xi^{k+1}_i \}_{i=1}^{N_i}$ so we conclude that $\Pr\{\exists \ \vb \text{ such that } W_{k+1}\vb  =0 \big| \xi^{[k]}\} = 0, \forall \ k \ge 0$.
\end{proof}
Even though we have proved that $W_k$'s generated by the algorithm are full rank, we can not necessarily apply lemma \ref{lemma3.1} directly because it takes an arbitrary $W$ whereas $W_k$ is dependent on data $(u^i, v^i)$. We still prove that matrix $D$ generated along the trajectory of the algorithm is full rank. We use techniques inspired from lemma \ref{lemma3.1} but this time we use Lebesgue measure over $\Theta$ rather than data. Over randomness of $\Theta$, we can show that our algorithm will not produce any $W$ such that corresponding matrix $D$ is rank deficient. Since $\Theta$ is essentially designed to be independent of data so we will not produce rank deficient $D$ throughout the process of random iid data collection and randomized algorithm. Before we jump into proving that we give a supplementary lemma which shows a more general result about the rank of matrix $D$.
\begin{lemma} 
\label{lemma5.2}Suppose $W = W'+ D_v Z$ where $D_v := \text{diag}(v[i], i \in [d])$ and $v$ is a random vector with Lebesgue measure in $\Rbb^d$. $W',Z \in \Rbb^{d\times d}$ and $Z \ne \textbf{0}$. Let $h$ be a function which follows condition \textbf{C1}. Also assume that $W$ is full rank with measure 1 over randomness of $v$. Then $h(Wu^i)u^{i^T}$ is full rank with measure 1 over randomness of $v$.
\end{lemma}
\begin{proof}
We use induction on $d$. For $d=1$ this is trivially true. Now assume this is true for $d-1$. We will show this for $d$.\\
Let $z^i := Wu^i = W'u^i + D_v Zu^i$. For simplicity of notation define $t^i :=Zu^i$. Due to simple linear algebraic fact provided by full rank property of $W$ we have rank of collection $(h(Wu^i)u^{i^T} =$ rank of collection $h(z^i)z^{i^T}$. For the sake of contradiction, say the collection is rank deficient with positive probability then there exists $d$-dimensional volume $\VV$ such that for all $v \in \VV$, we have $h(Wu^i)u^{i^T}$ is not full rank where $W :=W(v) = W'+D_vZ$. Without loss of generality, we may assume $d$-dimensional volume to be a hypercuboid $\VV := \{x \in \Rbb^d| a<x<b\}$ (if not then we can inscribe a hypercuboid in that volume). Let us take $v\in \VV$ and $\eps$ small enough such that $\wh{v} := v+\eps e_1 \in \VV$. Correspondingly we have $z^i$ and $\wh{z}^i$. Note that $\wh{z}^i = z^i + \eps t^i[1]$. So in essence, a small $\eps$ change in $v[1]$ causes $\eps t^i[1]$ change in vector $z^i[1]$.\\
Let $\vb^i = \text{vect}(h(z^i)z^{i^T})$. Similarly, $\wh{\vb}^i = \text{vect}(h(\wh{z}^i)\wh{z}^{i^T})$. So we can divide $\vb^i = \begin{bmatrix}
c^i\\g^i
\end{bmatrix}$ such $c^i \in \Rbb^{2d-1}$ and $g^i \in \Rbb^{(d-1)^2}$. Here \[c^i := \begin{bmatrix}
h(z^i[1])z^i[1] \\ h(z^i[2])z^i[1] \\ \vdots \\ h(z^i[d])z^i[1] \\ h(z^i[1])z^i[2] \\ \vdots \\ h(z^i[1])z^i[d]
\end{bmatrix}, \quad g^i := \text{vect}(h(y^i)y^{i^T}), \quad y^i := z^i[2:d]\]
Similarly we also have $\wh{\vb}^i= \begin{bmatrix}
\wh{c}^i \\ \wh{g}^i
\end{bmatrix}$. Now by the act that $v, \wh{v}$ corresponding to $z, \wh{z}$ are in $\VV$, and our assumption of linear dependence for all $v \in \VV$ we get \begin{align}
\vb^1 &= \mu_2 \vb^2 + \dots + \mu_N \vb^N \label{linear_dependence_1}\\ 
\wh{\vb}^1 &= \wh{\mu}_2 \wh{\vb}^2 + \dots + \wh{\mu}_N \wh{\vb}^N \label{linear_dependence_2}
\end{align}
Now notice that $y^i = \wh{y}^i, \forall \ i \in [N]$. So $g^i = \wh{g}^i, \forall \ i \in [N]$. Also by induction on $d-1$, we have that the rank of collection $g^2, \dots, g^N \ge (d-1)^2$. So we can rewrite matrix $[g^2 \dots g^N] := [G\quad \wt{G}]$ such that $G \in \Rbb^{(d-1)^2 \times (d-1)^2}$ is an invertible matrix and rewrite one part of equation \eqnok{linear_dependence_1} as $g^1 = [G \quad \wt{G}]\begin{bmatrix}
\wt{\mu} \\ {\mu}
\end{bmatrix}$. Hence we can replace $\wt{\mu} = G^{-1}(g^1-\wt{G}\mu) = G^{-1}g^1 - G^{-1}\wt{G}\mu$. Essentially the vector $\begin{bmatrix}
\wt{\mu} \\ \mu
\end{bmatrix}$ is completely defined by parameter $\mu \in \Rbb^{N-1 - (d-1)^2}$. Similarly we have $\wt{\wh{\mu}} = G^{-1}g^1 - \wt{G}\wh{\mu}$, so vector $\begin{bmatrix}
\wt{\wh{\mu}}\\ \wh{\mu}
\end{bmatrix}$ is completely defined by $\wh{\mu} \in \Rbb^{N-1-(d-1)^2}$. So essentially we have satisfied one part of equations \eqnok{linear_dependence_1} and \eqnok{linear_dependence_2}. Notice that since we are moving only one coordinate of random vector $v$ i.e. $v[1] \in (a[1], b[1])$ (by $\eps$ incremental changes) keeping all other elements of $v$ constant so we will have $y^i$ as constants which implies $g^i, G, \wt{G}$ are constant. So for the sake of simplicity of notation we define $l :=G^{-1}g^1 \in \Rbb^{(d-1)^2}$ and $R:= G^{-1}\wt{G} \in \Rbb^{(d-1)^2 \times (N-1-(d-1)^2)}$\\
Now, we look at the remaining part of two equation \eqnok{linear_dependence_1},\eqnok{linear_dependence_2}:
\begin{align*}
c^1 &= \mu_2 c^2 + \dots + \mu_N c^N, \\ 
\wh{c}^1 &= \wh{\mu}_2 \wh{c}^2 + \dots + \wh{\mu}_N \wh{c}^N, 
\end{align*}
which can be rewritten as 
\begin{align}
c^1 = [C \quad \wt{C}]\begin{bmatrix} l-R\mu \\ \mu \end{bmatrix} = Cl -CR\mu +\wt{C}\mu \label{lin_depe_c},\\
\wh{c}^1 = [\wh{C} \quad \wh{\wt{C}}]\begin{bmatrix} l-R\wh{\mu} \\ \wh{\mu} \end{bmatrix} = \wh{C}l -\wh{C}R\wh{\mu} +\wh{\wt{C}}\wh{\mu}. \label{lin_depe_c_hat}
\end{align}
After \eqnok{lin_depe_c_hat} $-$ \eqnok{lin_depe_c}, we have
\begin{equation}\label{discrete_diff_c}
(\wh{C}-C)l -(\wh{C}-C)R\mu-\wh{C}R(\wh{\mu}-\mu)+ (\wh{\wt{C}}-\wt{C})\mu +\wh{\wt{C}}(\wh{\mu}-\mu) = \wh{c}^1-c^1.
\end{equation}
Now note that \eqnok{discrete_diff_c}, characterizes incremental changes in $C, \wt{C}, \mu$ due to $\eps$. So taking the limit as $\eps \to 0$, we have
\begin{align}
 c^{1'} &=C' l-C'R\mu -CR\mu'+ \wt{C}'\mu + \wt{C}\mu'.  \nonumber \\
\begin{bmatrix}c^{1'} & C' \end{bmatrix}
\begin{bmatrix}
1\\ -l
\end{bmatrix} &=  (-CR+\wt{C})\mu'+ (-C'R+\wt{C}')\mu. \nonumber \\
\Rightarrow \begin{bmatrix}
c^1 & C
\end{bmatrix} \begin{bmatrix}
1 \\ -l
\end{bmatrix}
&= (-CR+\wt{C})\mu. \label{lin_depe_c_last}
\end{align}
Here, last equation is due to product rule in calculus. In \eqnok{lin_depe_c_last}, we see that we have $2d-1$ equations and $N-1-(d-1)^2$ unknowns at every point. If $N \le d^2$ then $N-1-(d-1)^2 \le 2d-2$. So at least one equation should depend on others. But as we have shown earlier, $h$ satisfying condition \textbf{C1} does not have row dependence. So we arrive at the required contradiction for $N \le d^2$. That completes the proof.
\end{proof}
\begin{lemma} Collection of matrices $h'(W_{k+1}u^i)u^{i^T}$ are full rank with measure 1, where the measure is over randomness of $\xi^{k+1}_{[N_i]}$
\end{lemma}
\begin{proof} We know that 
\[
W_{k+1} = W_k + \gamma_k\Theta_{k+1} \sum\limits_{j=1}^{N}h'(W_ku^j)u^{j^T}(v^i-\theta^Th(W_ku^j)).
\]
Now apply lemma \ref{lemma5.2} to obtain the required result.
\end{proof}
Hence we showed that algorithm will generate full rank matrix $D$ for any finite iteration.\\
Now to prove convergence of the algorithm, we need to analyze the function $f$ (defined in \eqnok{SAA_DL}) itself. We show that $f$ is a Lipschitz smooth function for any given instance of data $\{u^i, v^i\}_{i=1}^{N}$. This will give us a handle to estimate convergence rates for the given algorithm.
\begin{lemma} \label{lemma3.6}
Assuming that $h: \Rbb \to \Rbb$ is such that its gradients, hessian as well as values are bounded and data $\{u^i, v^i\}_{i=1}^{N}$ is given then there exists a constant $L$ such that
\begin{equation}\label{eqLipSmooth} 
\big\| \nabla_{W} f(W_1, \theta)-\nabla_{W} f(W_2, \theta)\big \|_F \le L \big\| W_1-W_2 \big\|_F.
\end{equation} Moreover, a possible upper bound on $L$ can be as follows:
\[
L \le  \frac{1}{N}\theta_{\max} \Big( L_{h'} \big( \sum\limits_{i=1}^{N} \|u^i \|^2_2 |v^i| \big)+ \sqrt{2d}L_{hh'} \|\theta \|_2 \big( \sum\limits_{i=1}^{N} \| u^i\|^2_2 \big) \Big)
\]
\end{lemma}
\begin{remark} Before stating the proof, we should stress that assumptions on $h$ is satisfied by most activation functions e.g., sigmoid, sigmoid symmetric, gaussian, gaussian symmetric, elliot, elliot symmetric, tanh, Erf.
\end{remark}
\begin{proof}
Assume that all the gradients in this proof are w.r.t. $W$ then we know that 
\[-\nabla f(W, \theta) [j,k] = \frac{1}{N}\sum\limits_{i=1}^{N}\{v^i-\theta^T h(W u^i )\}h'(W[j,:]u^i ) \theta[j] u^i[k] \] 
Notice that $\|W\|_F = \|\text{vect}(W)\|_2$. Also notice that if $W = ab^T$ then $\|W\|_F = \|a\|_2.\|b\|_2$\\ 
Let us define vector $a^{i}$ s.t. $a^{i}[j] = \theta[j] h'(W[j,:]u^{i})(v^i - \theta^T h(Wu^{i}))$ so we have 
\begin{align}
-(\nabla f(W_1)-\nabla f(W_2))_{jk} &= \frac{1}{N}\sum\limits_{i=1}^{N}u^{i}_k(a^i_1[j]-a^i_2[j]) \nonumber \\
\Rightarrow -(\nabla f(W_1)-\nabla f(W_2)) &= \frac{1}{N}\sum\limits_{i=1}^{N}(a^i_1-a^i_2)u^{i^T} \nonumber \\
\Rightarrow \|\nabla f(W_1)-\nabla f(W_2)\|_F &\le \frac{1}{N}\sum\limits_{i=1}^{N} \big\| u^i \big\|_2. \big\| a^i_1-a^i_2 \big\|_2, \label{first eq for smoothness}
\end{align}
where the last inequality follows from Cauchy-Schwarz inequality.\\
So if we can show Lipschitz constant  $L_i$ on $\big\| a^i_1-a^i_2 \big\|_2, \forall \ i$ then we are done. \\
 Let $\theta_{\max} := \max\limits_j |\theta_j|$, then 
 \begin{align*}
 \Big |(a^i_1-a^i_2)[j]\Big| &= |\theta_j|. \bigg| h'(W_1[j,:]u^i)(v^i - \theta^T h(W_1u^i))-h'(W_2[j,:]u^i)(v^i - \theta^T h(W_2u^i)) \bigg| \\
 & \le \theta_{\max} \bigg| h'(W_1[j,:]u^i)(v^i - \theta^T h(W_1u^i))-h'(W_2[j,:]u^i)(v^i - \theta^T h(W_2u^i)) \bigg| \\
 \Rightarrow \big\| a^i_1 - a^i_2 \big\|_2 &\le \theta_{\max} \bigg\|  v^i \Big( h'(W_1u^i)- h'(W_2u^i) \Big)- \Big( h(W_1u^i) h'(W_1u^i)^T- h(W_2u^i) h'(W_2u^i)^T \Big) \theta \bigg\|_2\\
 &\le \theta_{\max} \bigg\{ \Big\| v^i \Big( h'(W_1u^i)-h'(W_2u^i) \Big)\Big\|_2\\
 &\qquad+ \Big\|(h'(W_1u^{i})h(W_1u^{i})^T-h'(W_2u^{i})h(W_2u^{i})^T) \theta \Big\|_2\bigg\}.
 \end{align*}
 Suppose the Lipschitz constants for the first and second term are $L_{i,L}\text{ and }L_{i,R}$ respectively. Then $L_i = \theta_{\max} (L_{i,L} +L_{i,R})$ and possible upper bound on value of $L$ would become $\frac{1}{N}\sum\limits_{i=1}^{N} \|u^i\|_2L_i$. We now analyse existence of $L_{i,L}$\\
Since the Hessian of scalar function $h(\cdot)$ is bounded so we have $h'(x)$ is Lipschitz continuous with constant $L_{h'}$. Let $r_1, r_2$ be two row vectors then we claim $\|h'(r_1x) -h'(r_2x)\|_2 \le L_{h'} \big\| x \big\|_2 .\big\| r_1- r_2 \big\|_2 , \forall \ r_1,r_2$ because:
\[ \| h'(r_1x)-h'(r_2x) \|_2 \le L_{h'} \big| r_1x-r_2x \big| \le L_{h'} \| x\|_2 \|r_1-r_2\|_2  \]
From the relation above we have the following: \begin{align}
\big\| h'(W_1u^{i}) - h'(W_2u^{i}) \big\|_2^2 &= \sum\limits_{j=1}^{d}\Big( h'(W_1[j,:] u^i) - h'(W_2[j,:]u^i) \Big)^2 \nonumber \\
&\le L_{h'}^2 \big\| u^i \big\|_2^2 \sum\limits_{j=1}^{d}\big\| W_1[j,:] - W_2[j,:] \big\|_2^2 = L_{h'}^2 \big\| u^i \big\|^2_2 \big\| W_1-W_2 \big\|_F^2 \nonumber \\
\Rightarrow L_{i,L} &= L_{h'}\|u^{i}\|_2|v^i|. \label{eq LiL}
\end{align}
Now we focus our attention to second term. Notice the simple fact that 
\begin{equation}\label{norm_relation_vect_frobenius}
\|W_1-W_2\|_2 \le \|W_1- W_2\|_F = \|\text{vect}(W_1-W_2)\|_2.
\end{equation}
Define $\vb := W_1u^{i},  \ub:=W_2u^{i}$, then we have
\begin{equation}\label{vu_relation}
\| \vb - \ub \|_2 = \Big\|(W_1-W_2)u^{i}\Big\|_2 \le \Big\|W_1-W_2\Big\|_2.\Big\|u^{i}\Big\|_2 \le \Big\|u^{i}\Big\|_2. \Big\|\text{vect}(W_1-W_2)\Big\|_2,
\end{equation}
and
\begin{align*}
\Big\| \Big( h'(W_1u^{i})h(W_1u^{i})^T&-h'(W_2u^{i})h(W_2u^{i})^T \Big) \theta\Big\|_2\\
&\le\Big\| \theta \Big\|_2. \Big\|h'(W_1u^{i})h(W_1u^{i})^T-h'(W_2u^{i})h(W_2u^{i})^T\Big\|_2\\
&=\Big\| \theta \Big\|_2. \Big\|h'(\vb)h(\vb)^T-h'(\ub)h(\ub)^T\Big\|_2\\
&\le \Big\| \theta \Big\|_2. \Big\|h'(\vb)h(\vb)^T-h'(\ub)h(\ub)^T\Big\|_F.
\end{align*}
The latter inequality implies that
\begin{align*}
 \Big\|(h'(W_1u^{i})h(W_1u^{i})^T&-h'(W_2u^{i})h(W_2u^{i})^T) \theta \Big\|_2^2 \\
&\le \big\| \theta \big\|_2^2\bigg(\sum\limits_{i,j=1}^{d}h'(\vb[i])h(\vb[j]) - h'(\ub[i])h(\ub[j])\bigg)^2.
\end{align*}
Now let us define a 2-D function $H(x_1,x_2) = h(x_1)h'(x_2)$. Then $\nabla H(x_1,x_2) = \begin{bmatrix}
h'(x_1)h'(x_2)\\h(x_1)h''(x_2)
\end{bmatrix}$ so  under given assumptions, $\|\nabla H(\cdot)\|_2$ is bounded. Let that bound be $L_{hh'}$.\\ 
Now by mean value theorem, we have 
\begin{align}
H(x_1,x_2)-H(y_1,y_2) &=\nabla H(\xi)^T\{(x_1,x_2)-(y_1,y_2)\} \nonumber \\
\Rightarrow \Big| H(x_1,x_2)-H(y_1,y_2) \Big|^2 &\le \Big\| \nabla H(\xi) \Big\|_2^2. \Big\{(x_1-y_1)^2+ (x_2-y_2)^2\Big\} \nonumber \\
& \le L_{hh'}^2\Big\{(x_1-y_1)^2+ (x_2-y_2)^2\Big\}\nonumber \\
\text{So } \Big\| \Big\{ h'(W_1u^{i})h(W_1u^{i})^T&-h'(W_2u^{i})h(W_2u^{i})^T \Big\} \theta \Big\|_2^2 \nonumber \\
&\le \big\| \theta \big\|_2^2 \bigg(\sum\limits_{i,j=1}^{d}h'(\vb[i])h(\vb[j]) - h'(\ub[i])h(\ub[j])\bigg)^2 \nonumber \\
&\le \big\| \theta \big\|_2^2 \sum\limits_{i,j =1}^{d}L_{hh'}^2\big((\vb[i]-\ub[i])^2+ (\vb[j]-\ub[j])^2\big) \nonumber \\
&= 2dL_{hh'}^2 \big\| \theta \big\|_2^2\| \vb - \ub \|_2^2 \label{vu_result}
\end{align}
It then follows from \eqnok{norm_relation_vect_frobenius},\eqnok{vu_relation} and \eqnok{vu_result} that
\begin{align*}
\Big\|(h'(W_1u^{i})h(W_1u^{i})^T&-h'(W_2u^{i})h(W_2u^{i})^T) \theta \Big\|_2 \nonumber\\
 &\le \sqrt{2d}L_{hh'}\big\| \theta \big\|_2 .\big\| u^i \big\|_2. \big\| W_1-W_2 \big\|_F 
\end{align*} 
So you get that $L_{i,R} = \sqrt{2d}L_{hh'}\|\theta\|_2\|u^{i}\|_2$\\
Finally, using \eqnok{first eq for smoothness}, \eqnok{eq LiL} and \eqnok{vu_result}, we get a possible finite upper bound on the value of $L$:  
\[
L \le  \frac{1}{N}\theta_{\max} \Big( L_{h'} \big( \sum\limits_{i=1}^{N} \|u^i \|^2_2 |v^i| \big)+ \sqrt{2d}L_{hh'} \|\theta \|_2 \big( \sum\limits_{i=1}^{N} \| u^i\|^2_2 \big) \Big)
\]
Also note that this bound is valid even if $W$ is not a square matrix.
\end{proof}
\begin{remark}\label{remark_L}
Note that one can easily calculate value of $L$ given data and $\theta$. Moreover, if we put constraints on $\big\|\theta\big\|_2$ then $L$ is constant in every iteration of the algorithm 1. As mentioned in section 3, this will provide an easier way to analyse the algorithm.
\end{remark}
\begin{lemma} \label{lemma_lip_smooth_w}
Assuming that scalar function $h$ is such that $|h(\cdot)| \le u$ then there exists $L_\theta$ s.t. 
\begin{equation}\label{eqLipSmooth_w} 
\big\| \nabla_{w} f(W, \theta_1)-\nabla_{w} f(W, \theta_2)\big \|_2 \le L_\theta \big\| \theta_1 - \theta_2 \big\|_2
\end{equation}
\end{lemma}
\begin{proof}
Noting that
\[ -\nabla_{\theta}f(W, \theta) = \frac{1}{N} \sum\limits_{i=1}^{N} \{v^i - \theta^Th(Wu^i)\} h(Wu^i), \] 
we have 
\begin{align*}
\big\| \nabla_{\theta} f(W, \theta_1) &- \nabla_w f(W, \theta_2) \big\|_2\\
&= \bigg\| \frac{1}{N} \sum\limits_{i=1}^{N} \Big[ \{v^i - \theta_1^Th(Wu^i)\} h(Wu^i) - \{v^i - \theta_2^Th(Wu^i)\} h(Wu^i) \Big] \bigg\|_2\\
&= \bigg\| \frac{1}{N} \sum\limits_{i=1}^{N} \Big[ \{ -h(Wu^i)h(Wu^i)^T\theta_1 + h(Wu^i)h(Wu^i)^T\theta_2 \} \Big] \bigg\|_2\\
&= \bigg\| \frac{1}{N} \sum\limits_{i=1}^{N}  h(Wu^i)h(Wu^i)^T(\theta_2 - \theta_1) \bigg\|_2\\
&\le \bigg\| \frac{1}{N} \sum\limits_{i=1}^{N}  h(Wu^i)h(Wu^i)^T \bigg\|_2 . \big\| \theta_1 - \theta_2 \big\|_2 \\
&= \frac{1}{N} \Big\|  \sum\limits_{i=1}^{N}  h(Wu^i)h(Wu^i)^T \Big\|_2 . \big\| \theta_1 - \theta_2 \big\|_2 \\
&= \frac{1}{N} \lambda_{\max} \Big( \sum\limits_{i=1}^{N}  h(Wu^i)h(Wu^i)^T \Big) . \big\| \theta_1 - \theta_2 \big\|_2\\
&\le \frac{1}{N} \bigg\{ \sum\limits_{i=1}^{N} \lambda_{\max} \Big( h(Wu^i)h(Wu^i)^T \Big) \bigg\} . \big\| \theta_1 -\theta_2 \big\|_2 &\because \text{Weyl's Inequality}\\
&= \frac{1}{N} \bigg\{ \sum\limits_{i=1}^{N} \big\| h(Wu^i) \big\|^2_2 \bigg\} . \big\| \theta_1 - \theta_2 \big\|_2\\
&\le u^2 d \big\| \theta_1 - \theta_2 \big\|_2
\end{align*}
where $u_1$ and $u_2$ are upper bounds on scalar functions $|h(\cdot)|$ and $|h'(\cdot)|$ respectively and $d$ is row-dimension of $W$.
\end{proof}
Notice that Lemma \ref{lemma_lip_smooth_w} gives us value of $L_\theta$ irrespective of value of $W$ or data. Also observe that $f(W, \cdot)$ is convex function since hessian \[\nabla^2_\theta f(W, \theta) = \frac{1}{N} \sum\limits_{i=1}^{N} h(Wu^i) h(Wu^i)^T,\] which is the sum of positive semidefinite matrices. By Lemma \ref{lemma_lip_smooth_w}, we know that $f(W,\cdot)$ is smooth as well. So we can use following convergence result provided by \cite{L12} for stochastic composite optimization. A simplified proof can be found in appendix.
\begin{theorem} \label{convergence_thm_w} 
Assume that stepsizes $\beta_i$ satisfy $0 < \beta_i \le 1/2L_\theta, \forall \ i \ge 1$. Let $\{\theta^{av}_{i+1}\}_{i\ge 1}$ be the sequence computed according to Algorithm 1. Then we have,
\begin{equation}
\Ebb[ f(W_k, \theta^{av}_{i+1}) - f(W_k, \theta_{W_k}^*)] \le K_0(i), \ \forall \ i \ge 1, \forall \ k \ge 0,
\end{equation}
where $K_0(i) := \Big( \sum\limits_{\tau=1}^{i}\beta_\tau \Big)^{-1} \bigg[ \big\| \wbar{\theta}_1 - \theta^*_{W_k} \big \|^2_2 + \sigma^2 \sum\limits_{\tau =1}^{i}\beta_i^2 \bigg]$
where $\wbar{\theta}_1$ is the starting point for inner iteration and $\sigma$ is defined in (\ref{eq_sigma}). 
\end{theorem}
Now we look at a possible strategy of selecting stepsize $\beta_i$. Suppose we adopt a constant stepsize policy then we have $\beta_i = \beta, \forall \ i \in [N_i]$. Then we have 
\[ \Ebb[ f(W_k, \theta^{av}_{N_i+1}) - f(W_k, \theta_{W_k}^*)] \le \frac{\big\| \wbar{\theta}_1 - \theta^*_{W_k} \big \|^2}{N_i \beta} + \sigma^2 \beta.\]
Now if we choose 
\begin{equation} \label{eq_inner_stepsize}
\beta = \min \Big\{ \frac{1}{2L_\theta}, \sqrt{\frac{1}{N_i \sigma^2}}\Big\},
\end{equation} we get 
\[ \Ebb[ f(W_k, \theta^{av}_{N_i+1}) - f(W_k, \theta_{W_k}^*)] \le \big\| \wbar{\theta}_1 - \theta^*_{W_k} \big \|^2 \Big[ \frac{2L_\theta}{N_i} + \frac{\sigma}{\sqrt{N_i}} \Big] +  \frac{\sigma}{\sqrt{N_i}}.\]
By Lemma \ref{lemma3.6}, the objective function for neural networks is Lipschitz-smooth with respect to the hidden layer, i.e., it satisfies eq (\ref{eqLipSmooth}). Notice that it is equivalent to saying 
\begin{equation}\label{eqLipSmooth2}
\Big| f(W_2, w) - f(W_1,w) -\inprod{\nabla_W f(W_1,w)}{W_2-W_1} \Big| \le \frac{L}{2} \Big\| W_1- W_2 \Big\|_F^2, \quad
\forall\ W_1, W_2 \in \Rbb^{d\times d}.
\end{equation}
Since we have a handle on the smoothness of objective function, we can provide a convergence result for the overall algorithm.
\begin{theorem} \label{theorem_conv_W}
Suppose $\gamma_k < \frac{2}{L}$ then we have 
\begin{equation}\label{gradient_thm1}
\Ebb \bigg[ \min\limits_{k=0,\dots,N} \big\|\nabla f_{W}(W_k, \theta_{k+1})\big\|_F^2 \bigg] \le \frac{ f(W_0,\theta_0) + \sum\limits_{k=0}^{N_o}\big( \sum\limits_{\tau = 1}^{N_i} \beta^k_\tau \big)^{-1}\bigg[ \frac{R^2}{2} + \sum\limits_{\tau=1}^{N_i}\frac{\beta^{k^2}_\tau \sigma^2}{2(1-L_\theta\beta^k_\tau)} \bigg] } {\sum\limits_{k=0}^{N_o}(\gamma_k -L/2\gamma_k^2)},
\end{equation} 
where $R/2$ is the radius of origin centred ball, $\RR$ in algorithm, defined as $ \RR := \{r \in \Rbb^d : \|r\|_2 \le \frac{R}{2}\}$.
\end{theorem}
\begin{proof}
We know by lemma \ref{lemma3.6} that $f(\cdot, \theta)$ is a Lipschitz smooth function. So using (\ref{eqLipSmooth2}) we have 
\begin{align} 
f(W_{k+1}, \theta_{k+1}) &\le f(W_k,\theta_{k+1}) + \frac{L}{2} \Big\| \text{vect}(W_{k+1}-W_k)\Big\|^2\nonumber \\
&+ \inprod{\vect{\nabla_W f(W_k,\theta_{k+1})}}{\vect{W_{1_{k+1}}-W_{1_k}}} \nonumber \\
&=f(W_k,\theta_{k+1}) - \big( \gamma_k- \frac{L}{2} \gamma_k^2 \big) \Big\| \vect{\nabla_W f(W_k, \theta_{k+1})} \Big\|^2 \nonumber \\
&\le f(W_k,\theta_k) + \big( \sum\limits_{\tau=1}^{N_i}\beta^{k}_\tau \big)^{-1}\Bigg[ \frac{1}{2} \big\| \theta_k -\theta^*_{W_k} \big\|_2^2 + \sum\limits_{\tau=1}^{N_i}\beta^k_\tau \inprod{\xi^k_\tau}{\theta^*_{W_k}-\wbar{\theta}_\tau^k} \nonumber \\
 &+ \sum\limits_{\tau=1}^{N_i}\frac{\beta^{k^2}_\tau \big\| \xi^k_\tau \big\|^2}{2(1-L_\theta\beta^k_\tau)} \Bigg]- (\gamma_k- \frac{L}{2} \gamma_k^2)\Big\| \text{vect}(\nabla_{W}G(W_k, \theta_{k+1}))\Big\|^2, \label{conv_k_eq}
\end{align}
where the last inequality follows from equation (\ref{eq_converge_theta_1}) and (\ref{eq_converge_theta_2}).\\
From \eqnok{eq_define_RR}, we have $\|\theta \| \le R/2$ so $L$ is constant for each outer iteration. Summing (\ref{conv_k_eq}) from $k=0$ to $N_o$ and dividing both side by $\sum\limits_{k=0}^{N_o}(\gamma_k - \frac{L}{2}\gamma_k^2)$, we get 
\begin{align*}
\min\limits_{k= 0, \dots, N}\Big\| \nabla_{W}&f(W_k, \theta_{k+1})\Big\|_F^2 \le  \sum\limits_{k=0}^{N_o}\frac{(\gamma_k- \frac{L}{2} \gamma_k^2)\Big\| \vect{\nabla_{W_{1}}f(W_k, \theta_{k+1})} \Big\|^2}{\sum\limits_{k=0}^{N_o}(\gamma_k - \frac{L}{2}\gamma_k^2)}\\
&\le \frac{ f(W_0,\theta_0) + \sum\limits_{k=0}^{N_o} \big( \sum\limits_{\tau = 1}^{N_i} \beta^k_\tau \big)^{-1}\Bigg[ \frac{R^2}{2} + \sum\limits_{\tau=1}^{N_i}\bigg\{ \beta^k_\tau \inprod{\xi^k_\tau}{\theta^*_{W_k}-\wbar{\theta}_\tau^k} + \frac{\beta^{k^2}_\tau \big\| \xi^k_\tau \big\|^2}{2(1-L_\theta\beta^k_\tau)} \bigg\} \Bigg] } {\sum\limits_{k=0}^{N_o}(\gamma_k -L/2\gamma_k^2)}.
\end{align*}
Now taking expectation with respect to $\xi^{[N_o]}$ (which is defined in (\ref{eq_def_xi_outer})), we have 
\[\Ebb \bigg[ \inprod{\xi^k_\tau}{\theta^*_{W_k}-\wbar{\theta}_\tau^k} \Big|{\xi^{[k-1]} \cup \xi^k_{[\tau-1]}} \bigg] = 0,\] 
which implies $\Ebb_{\xi^{[N_o]}} \Big[ \inprod{\xi^k_\tau}{\theta^*_{W_k}-\wbar{\theta}_\tau^k} \Big] = 0$. We also have $\Ebb_{\xi^{[N_o]}}\Big[ \big\| \xi^k_\tau \big\|^2 \Big] \le \sigma^2$, and hence
\[ \Ebb \Big[ \min\limits_{k= 0, \dots, N}\Big\| \nabla_{W} f(W_k, \theta_{k+1})\Big\|_F^2 \Big] \le \frac{ f(W_0,\theta_0) + \sum\limits_{k=0}^{N_o}\big( \sum\limits_{\tau = 1}^{N_i} \beta^k_\tau \big)^{-1}\bigg[ \frac{R^2}{2} + \sum\limits_{\tau=1}^{N_i}\frac{\beta^{k^2}_\tau \sigma^2}{2(1-L_\theta\beta^k_\tau)} \bigg] } {\sum\limits_{k=0}^{N_o}(\gamma_k -L/2\gamma_k^2)}. \]
\end{proof}
In view of theorem \ref{theorem_conv_W}, we can derive a possible way of choosing $\gamma_k, \sigma$ and $N_i$ to obtain a convergence result. More specifically, if $N_i = N_o, \sigma = \frac{1}{\sqrt{N_i}}, \gamma_k = \frac{1}{L}$ and $\beta^k_\tau$ is chosen according to (\ref{eq_inner_stepsize}) then we have 
\[ \Ebb \bigg[ \min\limits_{k= 0, \dots, N}\Big\| \nabla_{W}f(W_k, \theta_{k+1})\Big\|_F^2 \bigg] \le \frac{2L\Big( f(W_0,\theta_0) + R^2( L_\theta + 1/2) +1 \Big)}{N_o}\]
Note that since we prove Lipschitz smoothness of objective function, $f(\cdot, \theta)$, we can apply whole host of the algorithms developed in literature for non-convex Lipschitz smooth objective minimization. More specifically, accelerated gradient method such as unified accelerated method proposed by \cite{GLZ15} or accelerated gradient method by \cite{GL16} can be applied in outer iteration. We can also use stochastic gradient descent method for outer iteration. For this, we need a stochastic algorithm that is designed for non-convex and Lipschitz smooth function optimization. Randomized stochastic gradient method, proposed by \cite{GL13}, Stochastic variance reduction gradient method (SVRG) by \cite{RHSPS16} or Simplified SVRG by \cite{AH16} can be employed in outer iteration. Convergence of these new algorithms will follow immediately from the convergence results of respective studies.\\
Value of Lipschitz constant, $L$, puts a significant impact on the running time of the algorithm. Notice that if $L$ increases then correspondingly $N_o$ and $N_i$ increase linearly with $L$. So we need methods by which we can reduce the value of the estimate of $L$. One possible idea would be to use $l_1$-ball for feasible region of $\theta$. More specifically, if $\RR = \BB^1(R/2)$ then we can possibly enforce sparsity on $\theta$ which will allow us to put better bound on $L$.

\bibliographystyle{acm}
\bibliography{paper_new}

\begin{thebibliography}{10}

\bibitem{AH16}
{\sc Allen-Zhu, Z., and Hazan, E.}
\newblock Variance reduction for faster non-convex optimization.
\newblock In {\em Proceedings of the 33rd International Conference on
  International Conference on Machine Learning - Volume 48}, ICML'16,
  pp.~699--707.

\bibitem{B94}
{\sc Barron, A.~R.}
\newblock Approximation and estimation bounds for artificial neural networks.
\newblock {\em Machine Learning\/} (1994), 115--133.

\bibitem{BR88}
{\sc Blum, A., and Rivest, R.~L.}
\newblock Training a 3-node neural network is np-complete.
\newblock In {\em Proceedings of the First Annual Workshop on Computational
  Learning Theory\/} (1988), COLT '88, pp.~9--18.

\bibitem{BG17}
{\sc Brutzkus, A., and Globerson, A.}
\newblock Globally optimal gradient descent for a convnet with gaussian inputs.
\newblock {\em CoRR\/} (2017).

\bibitem{CHMGY14}
{\sc Choromanska, A., Henaff, M., Mathieu, M., Arous, G.~B., and LeCun, Y.}
\newblock The loss surface of multilayer networks.

\bibitem{CW08}
{\sc Collobert, R., and Weston, J.}
\newblock A unified architecture for natural language processing: Deep neural
  networks with multitask learning.
\newblock In {\em Proceedings of the 25th International Conference on Machine
  Learning\/} (2008), ICML '08, pp.~160--167.

\bibitem{GL13}
{\sc Ghadimi, S., and Lan, G.}
\newblock Stochastic first- and zeroth-order methods for non-convex stochastic
  programming.
\newblock {\em SIAM Journal on Optimization\/} (2013), 2341--2368.

\bibitem{GL16}
{\sc Ghadimi, S., and Lan, G.}
\newblock Accelerated gradient methods for nonconvex nonlinear and stochastic
  programming.
\newblock {\em Math. Program. 156\/} (2016), 59--99.

\bibitem{GLZ15}
{\sc Ghadimi, S., Lan, G., and Zhang, H.}
\newblock Generalized uniformly optimal methods for nonlinear programming.
\newblock {\em CoRR\/} (2015).

\bibitem{HV15}
{\sc Haeffele, B.~D., and Vidal, R.}
\newblock Global optimality in tensor factorization, deep learning, and beyond.
\newblock {\em CoRR\/} (2015).

\bibitem{HLS15}
{\sc Hazan, E., Levy, K.~Y., and Shalev-Shwartz, S.}
\newblock Beyond convexity: Stochastic quasi-convex optimization.
\newblock In {\em Proceedings of the 28th International Conference on Neural
  Information Processing Systems - Volume 1\/} (2015), pp.~1594--1602.

\bibitem{KKKS11}
{\sc Kakade, S., Kalai, A.~T., Kanade, V., and Shamir, O.}
\newblock Efficient learning of generalized linear and single index models with
  isotonic regression.
\newblock {\em CoRR\/} (2011).

\bibitem{KS09}
{\sc Kalai, A., and Sastry, R.}
\newblock The isotron algorithm: High-dimensional isotonic regression.

\bibitem{KSH12}
{\sc Krizhevsky, A., Sutskever, I., and Hinton, G.~E.}
\newblock Imagenet classification with deep convolutional neural networks.
\newblock In {\em Proceedings of the 25th International Conference on Neural
  Information Processing Systems - Volume 1\/} (2012), NIPS'12, pp.~1097--1105.

\bibitem{L12}
{\sc Lan, G.}
\newblock An optimal method for stochastic composite optimization.
\newblock {\em Math. Program. 133}, 1-2 (2012), 365--397.

\bibitem{LY17}
{\sc Li, Y., and Yuan, Y.}
\newblock Convergence analysis of two-layer neural networks with relu
  activation.
\newblock {\em CoRR\/} (2017).

\bibitem{SYM17}
{\sc Mei, S., Bai, Y., and Montanari, A.}
\newblock The landscape of empirical risk for non-convex losses.

\bibitem{MDH12}
{\sc Mohamed, A., Dahl, G.~E., and Hinton, G.}
\newblock Acoustic modeling using deep belief networks.
\newblock {\em Trans. Audio, Speech and Lang. Proc.\/} (2012), 14--22.

\bibitem{NJLS09}
{\sc Nemirovski, A., Juditsky, A., Lan, G., and Shapiro, A.}
\newblock Robust stochastic approximation approach to stochastic programming.
\newblock {\em SIAM J. on Optimization 19\/} (2009), 1574--1609.

\bibitem{NBMS17}
{\sc Neyshabur, B., Bhojanapalli, S., McAllester, D., and Srebro, N.}
\newblock A pac-bayesian approach to spectrally-normalized margin bounds for
  neural networks.
\newblock {\em CoRR\/} (2017).

\bibitem{NH17}
{\sc Nguyen, Q.~N., and Hein, M.}
\newblock The loss surface of deep and wide neural networks.
\newblock {\em CoRR\/} (2017).

\bibitem{RHSPS16}
{\sc Reddi, S.~J., Hefny, A., Sra, S., P\'{o}cz\'{o}s, B., and Smola, A.}
\newblock Stochastic variance reduction for nonconvex optimization.
\newblock In {\em Proceedings of the 33rd International Conference on
  International Conference on Machine Learning - Volume 48}, ICML'16,
  pp.~314--323.

\bibitem{SJL17}
{\sc Soltanolkotabi, M., Javanmard, A., and Lee, J.~D.}
\newblock Theoretical insights into the optimization landscape of
  over-parametrized shallow neural networks.
\newblock {\em CoRR\/} (2017).

\bibitem{SC16}
{\sc Soudry, D., and Carmon, Y.}
\newblock No bad local minima: Data independent training error guarantees for
  multilayer neural networks.
\newblock {\em CoRR\/} (2016).

\bibitem{XLS16}
{\sc Xie, B., Liang, Y., and Song, L.}
\newblock Diversity leads to generalization in neural networks.
\newblock {\em CoRR\/} (2016).

\bibitem{CSMBO16}
{\sc Zhang, C., Bengio, S., Hardt, M., Recht, B., and Vinyals, O.}
\newblock Understanding deep learning requires rethinking generalization.
\newblock {\em CoRR\/} (2016).

\end{thebibliography}

\appendix 
\section{Proofs of Auxiliary Results}
In this appendix, we provide proofs for auxiliary results.
\subsection{Proof of Theorem \ref{convergence_thm_w}}
For sake of simplicity of notation, we define $\fb(\cdot) := f(W_k,\cdot), \gb(\cdot):= \nabla \fb(\cdot) = \nabla_\theta f(W_k, \cdot)$ and $G_{W_k}(\wbar{\theta}_\tau, \xi^k_\tau) := \Gb_\tau$. Then from (\ref{stochastic_oracle}) and Algorithm 1 we get \begin{equation} \label{eq_rel_G_g}
\Gb_\tau = \gb(\wbar{\theta}_\tau) + \xi^k_\tau \end{equation}
Also define $d_\tau := \wbar{\theta}_{\tau+1} - \wbar{\theta}_\tau$.\\
Notice that $\wbar{\theta}_{\tau+1}$ is optimal solution to the problem 
\begin{equation} \label{eq_innerstep_optimization}
\min\limits_{u \in \Rbb^d} \beta_\tau\inprod{\Gb_\tau}{u-\wbar{\theta}_\tau} + \frac{1}{2}\big\| u -\wbar{\theta}_\tau \big\|_2^2
\end{equation} 
by simply writing first order necessary condition for problem (\ref{eq_innerstep_optimization}). Also we note that optimization function in (\ref{eq_innerstep_optimization}) is strongly convex with parameter 1. Then we have 
\begin{equation} \label{eq_strongconvexity_innerloop}
\beta_\tau \inprod{\Gb_\tau}{d_\tau} + \frac{1}{2}\big\| d_\tau \big\|_2^2 + \frac{1}{2} \big\| u - \wbar{\theta}_{\tau+1} \big\|^2_2 \le \beta_\tau \inprod{\Gb_\tau}{u - \wbar{\theta}_\tau} + \frac{1}{2}\big\| u -\wbar{\theta}_\tau \big\|_2^2
\end{equation}
We will use eq (\ref{eq_strongconvexity_innerloop}) along with smoothness and convexity of the function $\fb$ to get the final convergence result.
\begin{align}
\beta_\tau \fb(\wbar{\theta}_{\tau+1}) &\le \beta_\tau [ \fb(\wbar{\theta}_{\tau})+\inprod{\gb(\wbar{\theta}_\tau)}{d_\tau} + \frac{L_\theta}{2} \| d_\tau \|^2 ] &\because \text{smoothness} \nonumber\\
&= \beta_\tau[ \fb(\wbar{\theta}_{\tau})+\inprod{\gb(\wbar{\theta}_\tau)}{d_\tau} ] + \frac{1}{2} \| d_\tau \|^2 -\frac{(1-L_\theta \beta_\tau)}{2} \| d_\tau \|^2 \nonumber \\
&= \beta_\tau [ \fb(\wbar{\theta}_{\tau})+\inprod{\Gb_\tau}{d_\tau} ] - \beta_\tau \inprod{\xi^k_\tau}{d_\tau} + \frac{1}{2} \| d_\tau \|^2 -\frac{(1-L_\theta \beta_\tau)}{2} \| d_\tau \|^2 &\because (\ref{eq_rel_G_g}) \nonumber \\
&\le \beta_\tau [ \fb(\wbar{\theta}_{\tau})+\inprod{\Gb_\tau}{d_\tau} ]  + \frac{1}{2} \| d_\tau \|^2 -\frac{(1-L_\theta \beta_\tau)}{2} \| d_\tau \|^2 +  \beta_\tau \| \xi^k_\tau \|.\|d_\tau \| \nonumber \\
&\le \beta_\tau \fb(\wbar{\theta}_{\tau})+\Big[ \beta_\tau \inprod{\Gb_\tau}{d_\tau}   + \frac{1}{2} \| d_\tau \|^2 \Big] + \frac{ \beta_\tau^2 \| \xi^k_\tau \|^2}{2 (1-L_\theta \beta_\tau)} \nonumber \\
&\le \beta_\tau \fb(\wbar{\theta}_{\tau})+ \beta_\tau \inprod{\Gb_\tau}{u - \wbar{\theta}_\tau} + \frac{1}{2} \| u -\wbar{\theta}_\tau \|^2 - \frac{1}{2} \| u -\wbar{\theta}_{\tau+1} \|^2 + \frac{ \beta_\tau^2 \| \xi^k_\tau \|^2}{2 (1-L_\theta \beta_\tau)} &\because (\ref{eq_strongconvexity_innerloop}) \nonumber \\
&= \Big[ \beta_\tau \fb(\wbar{\theta}_{\tau})+ \beta_\tau \inprod{\gb(\wbar{\theta}_\tau)}{u - \wbar{\theta}_\tau} \Big] + \beta_\tau \inprod{ \xi^k_\tau}{u - \wbar{\theta}_\tau} \nonumber\\ 
& + \frac{1}{2} \| u -\wbar{\theta}_\tau \|^2 - \frac{1}{2} \| u -\wbar{\theta}_{\tau+1} \|^2 + \frac{ \| \beta_\tau^2 \xi^k_\tau \|^2}{2 (1-L_\theta \beta_\tau)} \nonumber \\
&\le \beta_\tau \fb(u) + \beta_\tau \inprod{ \xi^k_\tau}{u - \wbar{\theta}_\tau} + \frac{1}{2} \| u -\wbar{\theta}_\tau \|^2 - \frac{1}{2} \| u -\wbar{\theta}_{\tau+1} \|^2 + \frac{ \beta_\tau^2 \| \xi^k_\tau \|^2}{2 (1-L_\theta \beta_\tau)} \label{eq_telescopic_theta}
\end{align}
Last equation is due to convexity of function $\fb$. So using (\ref{eq_telescopic_theta}) we get \begin{equation} \label{eq_converge_theta_1}
\sum\limits_{\tau = 1}^{i} \beta_\tau \big[ \fb(\wbar{\theta}_{\tau+1}) -\fb(\theta^*_{W_k}) \big] \le \frac{1}{2} \| \wbar{\theta}_1 - \theta^*_{W_k}\|^2 + \sum\limits_{\tau=1}^{i} \Big[ \beta_\tau \inprod{ \xi^k_\tau}{\theta^*_{W_k} - \wbar{\theta}_\tau} + \frac{ \beta_\tau^2 \| \xi^k_\tau \|^2}{2 (1-L_\theta \beta_\tau)} \Big] 
\end{equation}
Note that from convexity of $\fb$
\begin{equation}\label{eq_converge_theta_2} \fb(\theta^{av}_{i+1}) - \fb(\theta^*_{W_k}) \le \big( \sum\limits_{\tau=1}^{i}\beta_\tau \big)^{-1} \sum\limits_{\tau= 1}^{i}\Big[ \beta_\tau \big[ \fb(\wbar{\theta}_{\tau+1}) -\fb(\theta^*_{W_k}) \big]  \Big] \end{equation}
Moreover noting definition $\xi^{k}_{[\tau]}$ in (\ref{eq_def_xi_inner}) so we have,
\begin{equation}\label{eq_converge_theta_3} \Ebb \big[ \inprod{\xi^k_\tau}{\theta^*_{W_k} - \wbar{\theta}_\tau} \big| \xi^k_{[\tau-1]} \big] = 0 \end{equation}
and from (\ref{eq_sigma}) we get $\Ebb \big[ \|\xi^k_\tau\|^2 \big] \le \sigma^2$. So using this relation and noting $1- L_\theta \beta_\tau \ge \frac{1}{2}$, (\ref{eq_converge_theta_1}), (\ref{eq_converge_theta_2}) and (\ref{eq_converge_theta_3}) we get the result. 

\end{document}